\documentclass[10pt, onecolumn, a4paper]{article}

\usepackage[verbose=true,letterpaper]{geometry}
\makeatletter
\AtBeginDocument{
  \newgeometry{
    textheight=9in,
    textwidth=5.5in,
    top=1in,
    headheight=12pt,
    headsep=25pt,
    footskip=30pt
  }
}

\makeatother
\renewenvironment{abstract}%
{%
  \vskip 0.075in%
  \centerline%
  {\large\bf Abstract}%
  \vspace{0.5ex}%
  \begin{quote}%
}
{
  \par%
  \end{quote}%
  \vskip 1ex%
}

\usepackage{palatino}
\usepackage[utf8]{inputenc} 
\usepackage[T1]{fontenc}    
\usepackage{xcolor}
\usepackage{hyperref}
\ifdefined\nohyperref\else\ifdefined\hypersetup
  \definecolor{mydarkblue}{rgb}{0,0.08,0.45}
  \hypersetup{ %
    pdftitle={},
    pdfauthor={},
    pdfsubject={},
    pdfkeywords={},
    pdfborder=0 0 0,
    pdfpagemode=UseNone,
    colorlinks=true,
    linkcolor=mydarkblue,
    citecolor=mydarkblue,
    filecolor=mydarkblue,
    urlcolor=mydarkblue,
    pdfview=FitH}
\usepackage{url}            
\usepackage{booktabs}       
\usepackage{amsfonts}       
\usepackage{nicefrac}       
\usepackage{microtype}      
\usepackage{calc}
\usepackage{bm}
\usepackage{amssymb,amsmath}
\usepackage{mathtools}
\mathtoolsset{showonlyrefs}
\usepackage{wrapfig}
\usepackage{color}
\usepackage{caption}
\usepackage{soul}

\usepackage[compress]{natbib}

\usepackage{graphicx}
\usepackage{subfigure}
\usepackage{amsthm}

\usepackage{amsmath}
\usepackage{amssymb}
\usepackage{xcolor}
\usepackage{mathtools}
\usepackage{sidecap}
\usepackage{wrapfig}

\usepackage{enumitem}
\setlist[itemize,enumerate]{leftmargin=*}

\usepackage{mdframed}
\definecolor{theoremcolor}{rgb}{0.94, 0.94, 0.94}
\definecolor{examplecolor}{rgb}{1, 1, 1.0}
\mdfsetup{
    backgroundcolor=theoremcolor,
    linewidth=0pt,
}

\usepackage{xfrac}
\usepackage{comment}

\newmdtheoremenv[linewidth=0.5pt,innerleftmargin=4pt,innerrightmargin=4pt]{definition}{Definition}
\newmdtheoremenv[linewidth=0.5pt,innerleftmargin=4pt,innerrightmargin=4pt]{proposition}{Proposition}
\newmdtheoremenv[linewidth=0pt,innerleftmargin=0pt,innerrightmargin=0pt,backgroundcolor=examplecolor]{example}{Example}
\newmdtheoremenv{corollary}{Corollary}
\newmdtheoremenv{theorem}{Theorem}
\newmdtheoremenv{lemma}{Lemma}

\newcommand{\remove}[1]{}

\usepackage{fancyhdr}
\usepackage{titling}
\usepackage{authblk}

\usepackage{inconsolata}
\usepackage{tikz}

\usepackage{xspace}
\def\concrete{mixed\xspace}
\def\Concrete{Mixed\xspace}

\title{\Large\bf Reconciling the Discrete-Continuous Divide:\\ Towards a Mathematical Theory of Sparse Communication\\[2ex]}

\author{\normalsize Andr\'e F.~T.~Martins\textsuperscript{1,2,3}}

\affil{{\small\url{andre.t.martins@tecnico.ulisboa.pt}}}

\affil{\textsuperscript{1}%
  {\small Instituto de Telecomunica\c{c}\~oes, Instituto Superior Técnico, University of Lisbon, Portugal}
  } 
\affil{\textsuperscript{2}%
  {\small Lisbon ELLIS Unit of Machine Learning and Intelligent Systems (LUMLIS)} 
  }
\affil{\textsuperscript{3}%
  {\small Unbabel, Lisbon, Portugal}
  }

\date{\vspace{-5ex}}

\begin{document}

\maketitle

\begin{abstract}
Neural networks and other machine learning models compute continuous representations, while humans communicate with discrete symbols.
Reconciling these two forms of communication is desirable to generate human-readable interpretations or to learn discrete latent variable models, while maintaining end-to-end differentiability. 
Some existing approaches (such as the Gumbel-softmax transformation) build continuous relaxations that are discrete approximations in the zero-temperature limit, while others (such as sparsemax transformations and the hard concrete distribution) produce discrete/continuous hybrids.
In this paper, we build rigorous theoretical foundations for these hybrids. 
Our starting point is a new ``direct sum'' base measure defined on the face lattice of the probability simplex. 
From this measure, we introduce a new entropy function that includes the discrete and differential entropies as particular cases, and has an interpretation in terms of code optimality, as well as two other information-theoretic counterparts that generalize the mutual information and Kullback-Leibler divergences. 
Finally, we introduce ``\concrete languages'' as strings of hybrid symbols and a new \concrete weighted finite state automaton that recognizes a class of regular \concrete languages, generalizing closure properties of regular languages.
\end{abstract}

\section{Introduction}

%
%
%

As of today, 73 years have passed since the seminal  \textit{A Mathematical Theory of Communication} \citep{shannon1948mathematical} set the ground for modern information theory. The legacy of Claude E. Shannon's work is well visible in many contemporary disciplines, including Artificial Intelligence (AI), where the concepts of entropy, mutual information, and the noisy channel model are widely and routinely used. 

Seven decades fast forward, in a world where AI has an accelerated presence in our daily lives,  we are faced with many important challenges brought by the increasing interaction between humans and machines. Some of these challenges lie in the frontier between \textbf{discrete} and \textbf{continuous} domains. Take ``explainable AI'' as an example: How to make a machine able to ``communicate'' a decision's rationale, a thought process, or an internal representation into a form that a human can understand? The best performing AI systems today are built with neural networks, which perform \textit{continuous} computation and produce \textit{continuous}  representations -- by operating in a continuous space, they can be conveniently trained with the gradient backpropagation algorithm \citep{linnainmaa1970representation,werbos1982applications,Rumelhart1988}. 
Humans, however, speak a ``discrete'' language, composed of sequences of symbols. How can these two worlds be brought together?

Historically, discrete and continuous domains have been considered separately in information theory, applied statistics, and engineering applications: in most studies, random variables and sources are either discrete or continuous, but not \textit{both}. 
In signal processing, one needs to opt between discrete (digital) and continuous (analog) communication, whereas analog signals can be converted into digital ones by means of sampling and quantization. 
The use of discrete random variables in neural latent variable models is appealing: they can lead to understandable and human-controllable systems. 
However, the problem of learning these networks is far from solved: existing approaches use gradient surrogates, estimators with high variance, or continuous relaxations (discussed in more detail in \S\ref{sec:why}) as a means to interface the discrete and continuous worlds, with somewhat limited success.  

Since discrete variables and their continuous relaxations are so prevalent, they deserve a rigorous mathematical study. Throughout, we will use the name \textbf{\concrete variable} to denote a hybrid variable that are partly discrete, partly continuous.%
\footnote{When I started to write this manuscript, my preferred name for these hybrid variables was ``concrete''. However, \citet{maddison2016concrete} coined that name first with a different meaning (from ``\ul{con}tinuous relaxations of dis\ul{crete} variables'') in the context of their proposed ``concrete distribution'', also known as the Gumbel-softmax. Therefore, I am using the name ``\concrete'' to avoid any possible confusion.} %
In this draft, we will take a first step into a rigorous study of \concrete variables and their properties. 
We will call communication through \concrete variables \textbf{sparse communication}. The goal of sparse communication is to retain the advantages of differentiable computation but still being able to represent and approximate discrete symbols. 

The main contributions are:
\begin{itemize}
\item We represent existing transformations and densities on the probability simplex in terms of its \textbf{face lattice}. This lattice provides a combinatorial characterization of the simplex faces and opens the door to characterize sparse mixed distributions (\S\ref{sec:simplex}).
\item Based on the face lattice, we provide a \textbf{direct sum measure} as an alternative to the Lebesgue measure (used for continuous random variables) and the counting measure (used for discrete random variables). The direct sum measure takes all faces into account when expressing density functions, avoiding the need for products of Dirac densities when expressing densities with point masses in the boundary of the simplex.
\item We use the direct sum measure to formally define \textbf{\concrete random variables} and to propose a new \textbf{direct sum entropy}, which decomposes as a sum of discrete and continuous (differential) entropies. We show that the direct sum entropy has an interpretation in terms of optimal code length to encode a point in the simplex, and we provide an expression for the maximum entropy written as a generalized Laguerre polynomial (\S\ref{sec:information_theory}).
\item We define \textbf{\concrete strings} and \textbf{\concrete languages}, which are composed of hybrids of discrete and continuous symbols, and \textbf{\concrete (weighted) finite state automata}. We define a class of \concrete regular languages which contains (discrete) regular languages, and we derive closure properties for this class (\S\ref{sec:concrete_languages}).
\end{itemize}

\subsection{Why sparse communication?}\label{sec:why}

We list below  problems and applications where sparse communication might be desirable.

\paragraph{Explainability.} With the increasing interaction between humans and AI systems, there is a need for trust and transparency. For example, the EU introduced a ``right to explanation'' in General Data Protection Right (GDPR), where humans are entitled to receive explanations for justifying decisions taken by a machine. There is a tremendous body of work in ``explainable AI'' \citep[\textit{inter alia}]{doshi2017towards,miller2019explanation,chari2020directions}, either focusing on obtaining human readable interpretations of black-box algorithms \citep{ribeiro2016should}, or by designing systems that are interpretable by construction \citep{rudin2019stop}. 
Some proposals advocate ``neuro-symbolic integration'', in which symbolic-based expert systems evolve into hybrid systems that employ both statistical and logical reasoning techniques \citep{mooney1989experimental,bader2005dimensions}, emphasizing the need for integrating discrete and continuous representations in a shared workspace. 
One avenue of research includes rationalizers \citep{lei2016rationalizing}, where a component generates a rationale that influences the final prediction and also serves as an explanation to the end user. This rationale is typically a discrete sequence (such as a bit vector indicating relevant words in a document), constrained to be informative and short. Its discreteness precludes the backpropagation of gradients, posing difficulties to train the system end-to-end. Existing approaches rely on score function estimation \citep{lei2016rationalizing}, continuous relaxations combined with the reparametrization trick \citep{bastings2019interpretable}, and sparse relaxations \citep{treviso2020explanation}. 

\paragraph{Emergent communication.} Fostered by the success of neural networks, there has been a recent surge of interest in understanding emergent communication between artificial agents -- which communication protocols do agents develop when they have to cooperate to perform a task? What happens under channel capacity constraints? Understanding the conditions under which language evolves in communities of artificial agents and which features emerge may shed light on human language evolution and may help improving machine-machine and human-machine communication in the long run \citep{lazaridou2020emergent}. 
Multi-agent communication can be continuous, where messages consist of continuous vectors, or discrete, where they are fixed or variable-size sequences of symbols  \citep{foerster2016learning}. 
Continuous communication channels can  be regarded as
additional differentiable layers in a larger architecture encompassing the agents' networks. 
Discrete communication, on the other hand, prevents backpropagating the gradients through the discrete symbols, and training these networks is typically done with reinforcement learning \citep{foerster2016learning}
or continuous relaxations \citep{havrylov2017emergence}. 
It is hypothesized that this ``discrete bottleneck'', by preventing agents from accessing each others' states, forces the emergence of symbolic protocols. 
A related idea, inspired by the Global Workspace Theory from cognitive neuroscience \citep{baars1993cognitive} posits that generalization can emerge through an architecture of neural modules if their training encourages them to communicate effectively via the bottleneck of a shared global workspace \citep{goyal2021coordination}. 
Sparse communication with mixed symbols, the topic studied in our paper, could potentially provide a third type of multi-agent communication (or communication among modules) in-between the continuous and discrete cases, leading to the emergence of  new protocols.

\paragraph{Neural memories, attention mechanisms, and reasoning.} Many models have been developed to couple neural networks with external memory resources, including attention mechanisms \citep{bahdanau2014neural}, memory networks \citep{sukhbaatar2015end}, neural Turing machines and differentiable computers \citep{graves2014neural, graves2016hybrid}, continuous relaxations of data structures \citep{grefenstette2015learning}, and differentiable interpreters and theorem provers \citep{bovsnjak2017programming,rocktaschel2017end}. 
These models have been proposed to enable neural networks to perform tasks which require some form of complex reasoning involving discrete structures, and they all ``relax'' symbolic manipulation by building continuous (and differentiable) counterparts of these structures, sometimes sparse \citep{Martins2016ICML,fusedmax,niculae2018sparsemap,peters2019sparse}. 
It is likely that sparse memories and attention mechanisms will play an important role to develop better inductive biases for deep learning of higher-level cognition \citep{goyal2020inductive}. 

\paragraph{Discrete latent variable models.} 
Discrete latent variable models are appealing to facilitate learning with less supervision, to leverage prior knowledge via structural bias (\textit{e.g.} when a particular discrete structure is used), and to build more compact and more interpretable models. 
A challenge with discrete latent variable models such as variational auto-encoders \citep{kingma2013auto,kingma2019introduction} is to differentiate through the latent variables, which may involve computing a large or combinatorial expectation. Existing strategies include the score function estimator (also called \textsc{reinforce}, \citealt{williams1992simple}) combined with strategies for variance reduction \citep{mnih2014neural}, the pathwise gradient estimation (reparametrization trick) combined with a continuous relaxation of the latent variables (such as the Gumbel-softmax distribution, \citealt{jang2016categorical,maddison2016concrete}), and a parametrization of the latent variable that enables sparse expectations \citep{correia2020efficient}. This results in continuous approximations of quantities that are inherently discrete, sometimes creating a discrete-continuous hybrid \citep{louizos2018learning}.


\section{The Probability Simplex and All Its Faces}\label{sec:simplex}

We assume throughout an alphabet with $K \ge 2$ symbols indexed by integers $[K] = \{1, \ldots, K\}$. 
It is common to use a one-hot vector representation to indicate a symbol in this alphabet, \textit{i.e.}, the $k\textsuperscript{th}$ symbol corresponds to the vector
\begin{equation}
\bm{e}_k := [0, \ldots, 0, \underbrace{1}_{\text{$k$\textsuperscript{th} entry}}, 0, \ldots, 0].
\end{equation} 
We denote by $\mathbb{R}^K$ the $K$-dimensional Euclidean space, and by $\triangle_{K-1} \subseteq \mathbb{R}^K$ the \textbf{probability simplex}, $\triangle_{K-1} := \{\bm{y} \in \mathbb{R}^K \mid \bm{y}\ge \mathbf{0}, \,\, \mathbf{1}^\top \bm{y} = 1 \}$, whose vertices are the $K$ one-hot vectors above. Each point $\bm{y} \in \triangle_{K-1}$ can be regarded as a vector of probabilities for the $K$ symbols, parametrizing a categorical distribution over $[K]$.%
\footnote{With some abuse of notation, in the sequel we sometimes identify $\bm{y}$, which parametrizes this categorical distribution $\mathrm{Cat}(\bm{y})$, with the categorical distribution itself.} %
The \textbf{support} of $\bm{y} \in \triangle_{K-1}$ is the set of symbols with nonzero probability, $\mathrm{supp}(\bm{y}) := \{k \in [K] \mid y_k > 0\}$. 
The set of categorical distributions with full support corresponds to the \textbf{relative interior} of the simplex, $\mathrm{ri}(\triangle_{K-1}) := \{\bm{y} \in \triangle_{K-1} \mid \mathrm{supp}(\bm{y}) = [K]\}$. A categorical distribution parametrized by $\bm{y} \in \triangle_{K-1} \setminus \mathrm{ri}(\triangle_{K-1})$ (\textit{i.e.}, in the boundary of the simplex) is called a \textbf{sparse} distribution. 

\subsection{Transformations from $\mathbb{R}^K$ to $\triangle_{K-1}$}\label{sec:transforms} 

In many situations, there is a need to convert a vector of real numbers $\bm{z} \in \mathbb{R}^K$ (scores for the several symbols, often called \textit{logits}) into a categorical distribution $\bm{y} \in \triangle_{K-1}$. This arises in multinomial logistic regression, in models with discrete latent variables, in the last layer of neural networks for multi-class classification, in attention mechanisms, and in reinforcement learning, when choosing the next action to perform. We next review the most common strategies to do this. 

\paragraph{Softmax.} 
The most popular choice is by far the softmax transformation \citep{bridle1990probabilistic}: $$\bm{y} = \mathrm{softmax}(\bm{z}) := \frac{\exp(\bm{z})}{\mathbf{1}^\top \exp(\bm{z})}.$$ Since the exponential function is strictly positive, 
it follows that the softmax transformation reaches only the relative interior $\mathrm{ri}(\triangle_{K-1})$, that is, it never returns sparse distributions.%
\footnote{The softmax transformation can be regarded as a regularized argmax problem over the simplex, using the Shannon entropy as the regularizer. The inability of reaching boundary points is due to a property of the Shannon entropy called ``essential smoothness.'' This property is relaxed with $\alpha$-entmax for $\alpha>1$, which uses a generalized entropy \citep{Tsallis1988}. See \citet{wainwright_2008} and \citet{blondel2020learning} for details.} %
To encourage more peaked distributions (but never sparse) it is common to add a \textbf{temperature} parameter $\beta > 0$, by defining $\mathrm{softmax}_\beta(\bm{z}) := \mathrm{softmax}(\beta^{-1}\bm{z})$. The limit case $\beta \rightarrow 0_+$ corresponds to the indicator vector for the \textbf{argmax}, which returns a one-hot distribution indicating the symbol with the largest score.%
\footnote{When there are ties, argmax returns a uniform distribution supported on the highest-scored symbols.} 
While the softmax transformation is differentiable (hence permitting end-to-end training with the gradient backpropagation algorithm), the argmax indicator function has zero gradients almost everywhere. With very small temperatures, it is common to incur numerical problems and slow training.

\paragraph{Top-$k$ softmax.}  When a sparse distribution is desired, one possible approach is to remove the tail, keeping only the $k\textsuperscript{th}$ largest scores, and to perform the softmax transformation on those. The value of $k$ can be fixed \citep{fan2018hierarchical,radford2019language} or it can be chosen based on a given percentile \citep{holtzman2019curious}. This top-$k$ softmax operation is still differentiable almost everywhere for $k>1$, with nonzero gradients for the $k$ largest coordinates of $\bm{z}$. This transformation has been used in several settings, for example in recurrent independent mechanisms \citep{goyal2019recurrent}. 

\begin{figure}[t]
\begin{center}
\input{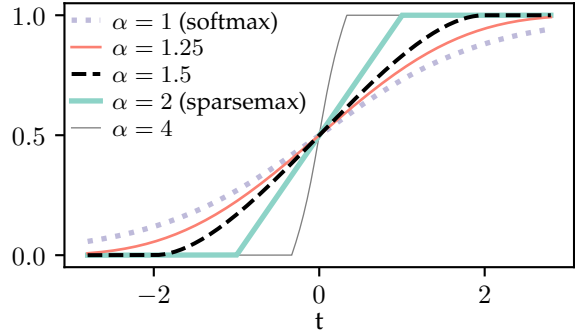}
\end{center}
\caption{\label{fig:entmax} Illustration of
entmax in the two-dimensional case $\alpha$-$\mathrm{entmax}([t, 0])_1$.
Softmax corresponds to $\alpha=1$ and sparsemax to $\alpha=2$. 
All transformations with $\alpha>1$ saturate at $t=\pm\nicefrac{1}{\alpha-1}$. 
Taken from \citet{peters2019sparse}.}
\end{figure}

\paragraph{Sparsemax and $\alpha$-entmax.} 
A more direct approach to achieve sparse distributions is \textit{sparsemax} \citep{Martins2016ICML}, the Euclidean projection onto the simplex, $$\mathrm{sparsemax}(\bm{z}) := \arg\min_{\bm{y} \in \triangle_{K-1}} \|\bm{y} - \bm{z}\|.$$  Unlike softmax, sparsemax reaches the \textit{full} simplex $\triangle_{K-1}$, including the boundary. With $K=2$ and fixing $z_2 = 0$ (without loss of generality), sparsemax becomes a ``hard sigmoid'' (see Figure~\ref{fig:entmax}). 
More generally, \textit{entmax} \citep{peters2019sparse} is a family of transformations parametrized by $\alpha\ge 0$, $$\text{$\alpha$-$\mathrm{entmax}$}(\bm{z}) := [1 + (\alpha-1)(\bm{z} - \tau \mathbf{1})]_+^{1/(\alpha-1)}.$$ This family recovers softmax as a limit case when $\alpha \rightarrow 1$ and sparsemax when $\alpha=2$. 
It is shown by \citet{blondel2020learning} that, for $\alpha > 1$, entmax also reaches the full simplex -- it can also return sparse distributions, and the coefficient $\alpha$ controls the propensity for sparsity. Entmax (and its particular case sparsemax) is differentiable almost everywhere, permitting efficient training. Sparsemax and entmax have been used as a component of neural networks to tackle several natural language processing tasks \citep{peters2019sparse,correia2019adaptively,martins2020sparse}. 


\subsection{Densities over the simplex}\label{sec:densities}


In \S\ref{sec:transforms}, we presented deterministic maps from $\bm{z} \in \mathbb{R}^K$ into $\bm{y} \in \triangle_{K-1}$. 
Let us now consider \textit{stochastic} maps. Denote by $Y$ a random variable taking values in the simplex $\triangle_{K-1}$ with probability density function $p_Y(\bm{y})$. We will see again that some transformations are limited to $\mathrm{ri}(\triangle_{K-1})$, while others are not.

\paragraph{Dirichlet.} 
The Dirichlet is a multivariate generalization of the Beta distribution, being the conjugate prior of the categorical distribution. It is widely used in topic models \citep{blei2003latent}. 
The density of a Dirichlet random variable $Y \sim \mathrm{Dir}(\bm{\alpha})$, with $\bm{\alpha} = [\alpha_1, \ldots, \alpha_K] > \mathbf{0}$, is
\begin{equation}\label{eq:dirichlet}
p_Y(\bm{y}; \bm{\alpha}) = \frac{1}{B(\bm{\alpha})}\prod_{k=1}^K y_k^{\alpha_k-1}, \quad \text{where $B(\bm{\alpha}) = \frac{\prod_{k=1}^K \Gamma(\alpha_k)}{\Gamma(\sum_{k=1}^K \alpha_k)}$}.
\end{equation}
When $\bm{\alpha} = \mathbf{1}$, this becomes a uniform distribution, with the constant density value $p_Y(\bm{y}; \mathbf{1}) = (K-1)!$.
Although a Dirichlet can assign high probability density to values of $\bm{y}$ close to the boundary of the simplex when $\bm{\alpha} < \mathbf{1}$, it is supported in $\mathrm{ri}(\triangle_{K-1})$ -- sampling from a Dirichlet will \textit{never} result in a $\bm{y}$ parametrizing a sparse categorical distribution. 

\paragraph{Logistic Normal.}
The logistic normal distribution \citep{atchison1980logistic} is an alternative to the Dirichlet which can capture correlations between labels \citep{blei2007correlated}. It can be described as the following generative story:%
\footnote{A minimal parametrization for the logistic normal considers $K-1$ Gaussian random variables only, followed by a softmax transformation where one the ``logits'' is fixed to zero, leading to a closed form density. We present here an overcomplete parametrization for simplicity.} %
\begin{align}
N &\sim \mathrm{Normal}(\mathbf{0}, \mathbf{I})\nonumber\\
Y &= \mathrm{softmax}(\bm{z} + \Sigma^{1/2}N). 
\end{align}
In words, $Y$ is the softmax transformation of a multivariate Gaussian random variable with mean $\bm{z}$ and covariance $\Sigma$. 
Again, since the softmax transformation is strictly positive, the logistic normal places no probability mass to points in the boundary of the simplex. 
By analogy to other distributions to be presented next, it is fair to call the logistic normal ``Gaussian-softmax.''

\paragraph{Gumbel-softmax.}
%
The Gumbel-softmax distribution \citep{jang2016categorical,maddison2016concrete}, also called \textit{concrete distribution},%
\footnote{It is worth noting that some of the names given to these distributions are inspired by the sampling procedure (Gaussian-softmax, Gumbel-softmax), while others are focused on the distribution and agnostic to the sampling (logistic normal, concrete). We will come back to different ways of sampling in the sequel.} %
has the following generative story:
\begin{align}
U_k &\sim \mathrm{Uniform}(0, 1), \quad k \in [K]\nonumber\\
G_k &= -\log(-\log(U_k)), \quad k \in [K]\nonumber\\
Y &= \mathrm{softmax}_\beta(\bm{z} + G). 
\end{align}
The name stems from the fact that the $G_k$ generated this way has a Gumbel distribution \citep{gumbel1935valeurs}. 
When the temperature $\beta$ approaches zero, the softmax approaches the indicator for argmax and $Y$ becomes closer to a discrete categorical random variable -- this reparametrization of a categorical distribution is known as the \textit{Gumbel-max} trick \citep{luce1959individual,papandreou2011perturb}. Thus, the Gumbel-softmax distribution can be seen as a \textit{continuous} relaxation of a categorical. Sampling from a Gumbel-softmax produces a point in the simplex $\triangle_{K-1}$. However, for positive $\beta$, this point will be in $\mathrm{ri}(\triangle_{K-1})$ with probability 1 -- like the previous cases, there is no probability mass assigned to the boundary. 
The following is an explicit density for the Gumbel-softmax distribution:
\begin{equation}\label{eq:concrete}
p_Y(\bm{y}; \bm{z}, \beta) = (K-1)! \,\, \beta^{K-1}\left( \sum_{k=1}^{K} \frac{\pi_k}{y_k^{\beta}} \right)^{-K}\prod_{k=1}^{K} \frac{\pi_k}{y_k^{\beta+1}}, \quad \text{with $\bm{\pi} = \mathrm{softmax}(\bm{z})$}.
\end{equation}
For $K=2$, by using the fact that the difference of two Gumbel random variables $G_1$ and $G_2$ is a logistic random variable $L=G_1-G_2$ \citep[Appendix B]{maddison2016concrete}, the generative story can be written as
\begin{align}
U &\sim \mathrm{Uniform}(0, 1)\\
L &= \log U - \log(1-U)\\
Y &= \mathrm{sigmoid}_\beta(z + L),
\end{align}
where $\mathrm{sigmoid}_\beta(z) = \frac{1}{1+\exp(-\beta^{-1}z)}$.

\paragraph{Binary hard concrete and HardKuma.} For $K=2$, a point in the simplex can be represented as $\bm{y} = (y, 1-y)$ and the simplex is isomorphic to the unit interval, $\triangle_1 \simeq [0,1]$. For this binary case, \citet{louizos2018learning} proposed a \textit{hard concrete distribution} which stretches the Gumbel-softmax \eqref{eq:concrete} and applies a hard sigmoid transformation (which equals the sparsemax with $K=2$) as a way of placing a point mass in the points $0$ and $1$. 
A more general case with $K\ge 2$ (never proposed before, to the best of our knowledge) would correspond to the generative story
\begin{align}
Y' &\sim \mathrm{GumbelSoftmax}(\bm{z}, \beta)\\
Y &= \mathrm{sparsemax}(\lambda Y'), \quad \text{with $\lambda \ge 1$}. 
\end{align}
A similar strategy, for $K=2$, underlies the HardKuma distribution \citep{bastings2019interpretable}, which instead of the Gumbel-softmax uses the Kumaraswamy distribution \citep{kumaraswamy1980generalized}.
These ``stretch-and-rectify'' techniques enable assigning probability mass to the boundary of $\triangle_1$, as shown in Figure~\ref{fig:rectify} (left). 
These techniques are similar in spirit to the spike-and-slab procedure for feature selection \citep{mitchell1988bayesian,ishwaran2005spike}. 

\begin{figure}[t]
\begin{center}
\includegraphics[width=.48\columnwidth]{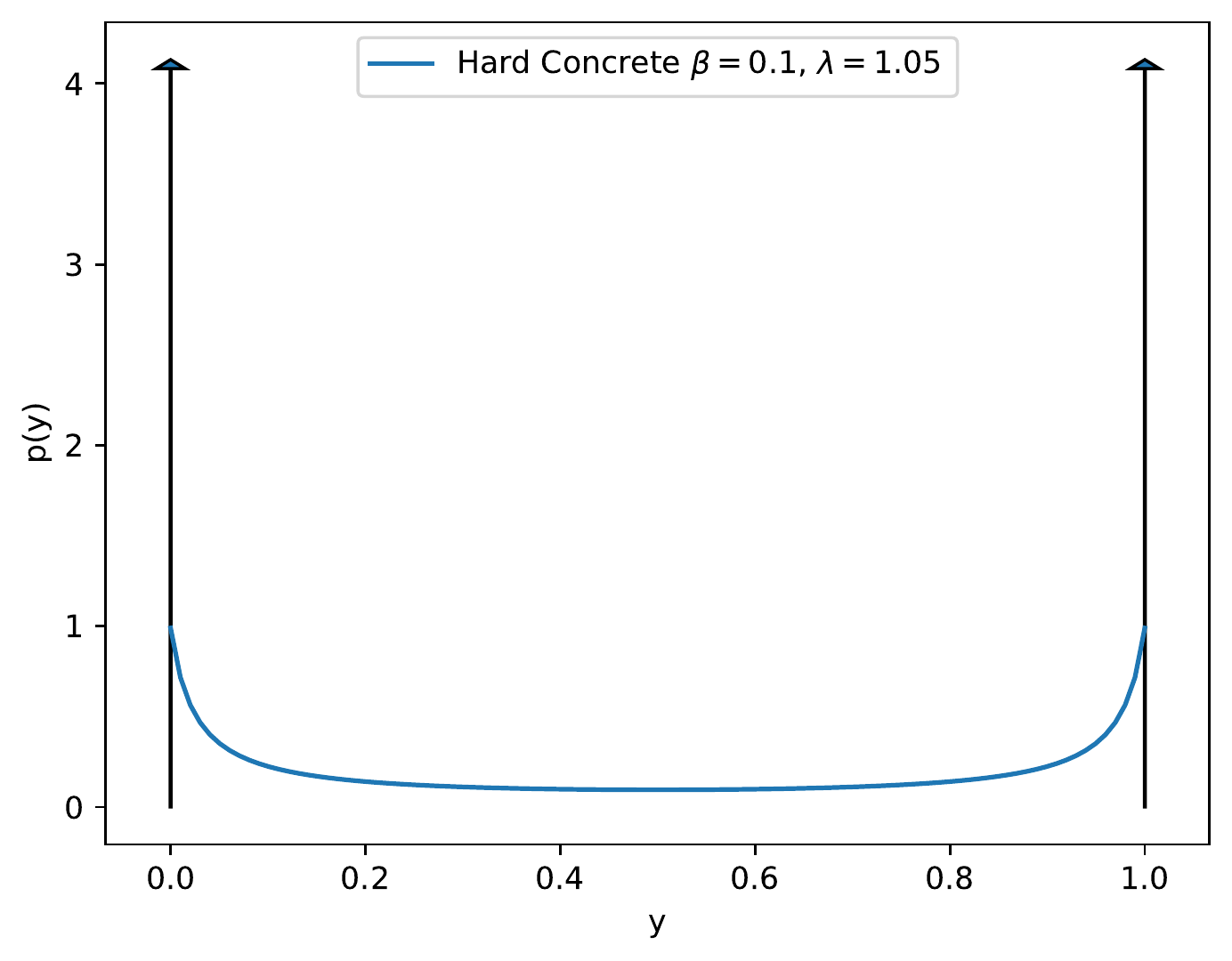}\,\,
\includegraphics[width=.48\columnwidth]{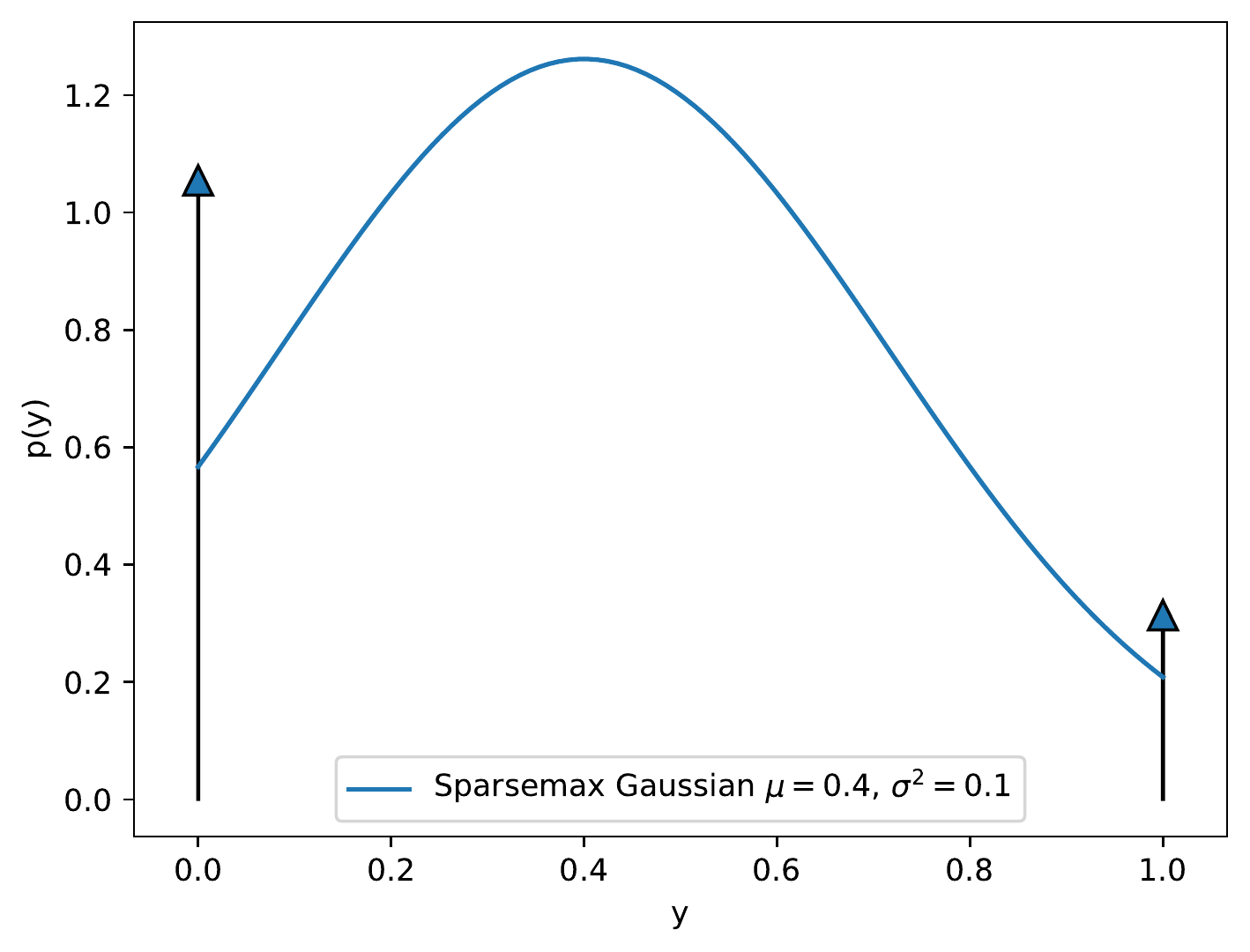}
\end{center}
\caption{\label{fig:rectify} \Concrete densities on $\triangle_1$, represented with Dirac deltas. Left: Hard Concrete with $\lambda=1.05$ and temperature $\beta=0.1$. Right: Gaussian sparsemax. The length of the Dirac arrows is 10 times the corresponding discrete probability value.}
\end{figure}

\paragraph{Gaussian-sparsemax.} It is possible to use the same rectification idea (but without any stretching required) to obtain a sparsemax counterpart of the logistic normal, which we call ``Gaussian-sparsemax.'' This has the following generative story:
\begin{align}
N &\sim \mathrm{Normal}(\mathbf{0}, \mathbf{I})\\
Y &= \mathrm{sparsemax}(\bm{z} + \Sigma^{1/2}N). 
\end{align}
Unlike the logistic normal, the Gaussian-sparsemax can assign non-zero probability mass to the boundary of the simplex (since the sparsemax transformation can lead to a sparse distribution). 
When $K=2$, writing $\bm{y} = (y, 1-y)$ and $\bm{z} = (z, 1-z)$, this distribution has the following density:
\begin{equation}\label{eq:gaussian_sparsemax}
p_Y(y) = \mathcal{N}(y; z, \sigma^2)  + \frac{1-\mathrm{erf}(z/(\sqrt{2}\sigma))}{2}\delta_0(y) + \frac{1+\mathrm{erf}((z-1)/(\sqrt{2}\sigma))}{2}\delta_1(y),
\end{equation}
where $\delta_s(y)$ is a Dirac delta density. This is illustrated in Figure~\ref{fig:rectify} (right).
For $K>2$, a density expression with Diracs would be cumbersome, since it would require a combinatorial number of Diracs of several ``orders,'' depending on whether they are placed at a vertex, edge, face, etc. 
Another annoyance is that Dirac deltas have $-\infty$ differential entropy, which prevents information theoretic treatment of these random variables. 
The next subsection shows how we can obtain densities that assign mass to the full simplex while avoiding Diracs, by making use of the face lattice and defining a new measure.

\paragraph{Extensions to the structured case.} 
Several of the approaches listed above have been extended more broadly to define distributions over structured and combinatorial variables, such as binary vectors representing trees and matchings. For example, a structured counterpart of sparsemax has been proposed by \citet{niculae2018sparsemap} under the name SparseMAP, and structured variants of Gumbel-like stochastic perturbation methods have been proposed by \citet{corro2018differentiable,berthet2020learning,paulus2020gradient}. 
Strategies for exploiting sparsity in combinatorial latent variables have also been considered by \citet{correia2020efficient}.

\subsection{Faces of the probability simplex}\label{sec:faces}

We next study the probability simplex in more detail, whose elements represent categorical distributions over $K$ elements.%
\footnote{Several discrete latent variable models consider configurations for the latent variables other than categorical distributions, for example products of independent Bernoulli random variables or structured latent variables which may have global constraints or higher-order parameters to model their interactions. We choose to study the probability simplex since it subsumes, for $K=2$, the Bernoulli case, and it is applicable to categorical distributions. It is possible to generalize  the concepts presented here, such as the face lattice, to arbitrary polytopes, which offers a natural way of extending this study to structured latent variables by using their marginal polytopes in place of the probability simplex \citep{wainwright_2008}.} %

The probability simplex $\triangle_{K-1}$ is an example of a \textbf{convex polytope} (\textit{i.e.}, a polyhedral set which is convex and bounded). As described by \citet[Lecture~2, \S2.2]{ziegler1995lectures} and \citep[\S3.2]{grunbaum2003convex}, the combinatorial structure of a polytope is determined by its {\bf face lattice}, which we now describe. 
Each \textbf{face} of the simplex $\triangle_{K-1}$ is induced by an index set $\mathcal{I} \subseteq [K]$. 
Given $\mathcal{I}$, the corresponding face $f_{\mathcal{I}}$ is the set
\begin{equation}
f_\mathcal{I} = \left\{\bm{p} \in \triangle_{K-1} \mid \mathrm{supp}(\bm{p}) \subseteq \mathcal{I}\right\}.
\end{equation}
This is the set of distributions assigning zero probability mass outside $\mathcal{I}$. 
We denote by $\mathcal{F}(\triangle_{K-1} )$ the set of \textit{all faces} of $\triangle_{K-1}$, which is isomorphic to the power set of $[K]$ (\textit{i.e.,} $\mathcal{F}(\triangle_{K-1} ) \simeq 2^{[K]}$). Note that  $\varnothing \in \mathcal{F}(\triangle_{K-1} )$; we denote by $\bar{\mathcal{F}}(\triangle_{K-1}) := \mathcal{F}(\triangle_{K-1}) \setminus \{\varnothing\}$ the set of faces excluding the empty set, and 
by $\mathrm{dim}(f_\mathcal{I}) := |\mathcal{I}| - 1$ the \textbf{dimension} of a face $f_\mathcal{I} \in \bar{\mathcal{F}}(\triangle_{K-1})$. Thus, 
vertices are $0$-dimensional faces, and the simplex $\triangle_{K-1}$ itself is a $(K-1)$-dimensional face, called the ``maximal face''. Any $k$-dimensional face can be regarded as a ``smaller'' simplex, \textit{i.e.}, $f_\mathcal{I} \simeq \triangle_{k-1}$ with $k = |\mathcal{I}|$. 

The set $\mathcal{F}(\triangle_{K-1} )$ has a partial order induced by set inclusion ($f_\mathcal{I} \subseteq f_\mathcal{J}$ iff $\mathcal{I} \subseteq \mathcal{J}$), that is, it is a partially ordered set (poset); more specifically it is a \textit{lattice}, hence the name \textit{face lattice}.
The full simplex $\triangle_{K-1}$ can be decomposed uniquely as the \textbf{disjoint union} of the relative interior of its faces:
\begin{equation}\label{eq:disjoint_union_faces}
\triangle_{K-1} = \bigsqcup_{f \in \bar{\mathcal{F}}(\triangle_{K-1})} \mathrm{ri}(f). 
\end{equation}
For example, the simplex $\triangle_2$ is composed of its face $\mathrm{ri}(\triangle_2)$ (\textit{i.e.}, excluding the boundary), three edges (excluding the vertices in the corners), and three vertices (the corners). This is represented schematically in Figure~\ref{fig:simplex_decomposition}. 
The partition \eqref{eq:disjoint_union_faces} implies that any subset $A \subseteq \triangle_{K-1}$ can be represented as a tuple 
$A = (A_f)_{f\in \bar{\mathcal{F}}(\triangle_{K-1})}$, where $A_f = A \cap \mathrm{ri}(f)$; and the sets $A_f$ are all disjoint. 

\begin{figure}[t]
\begin{center}
\includegraphics[width=1\columnwidth]{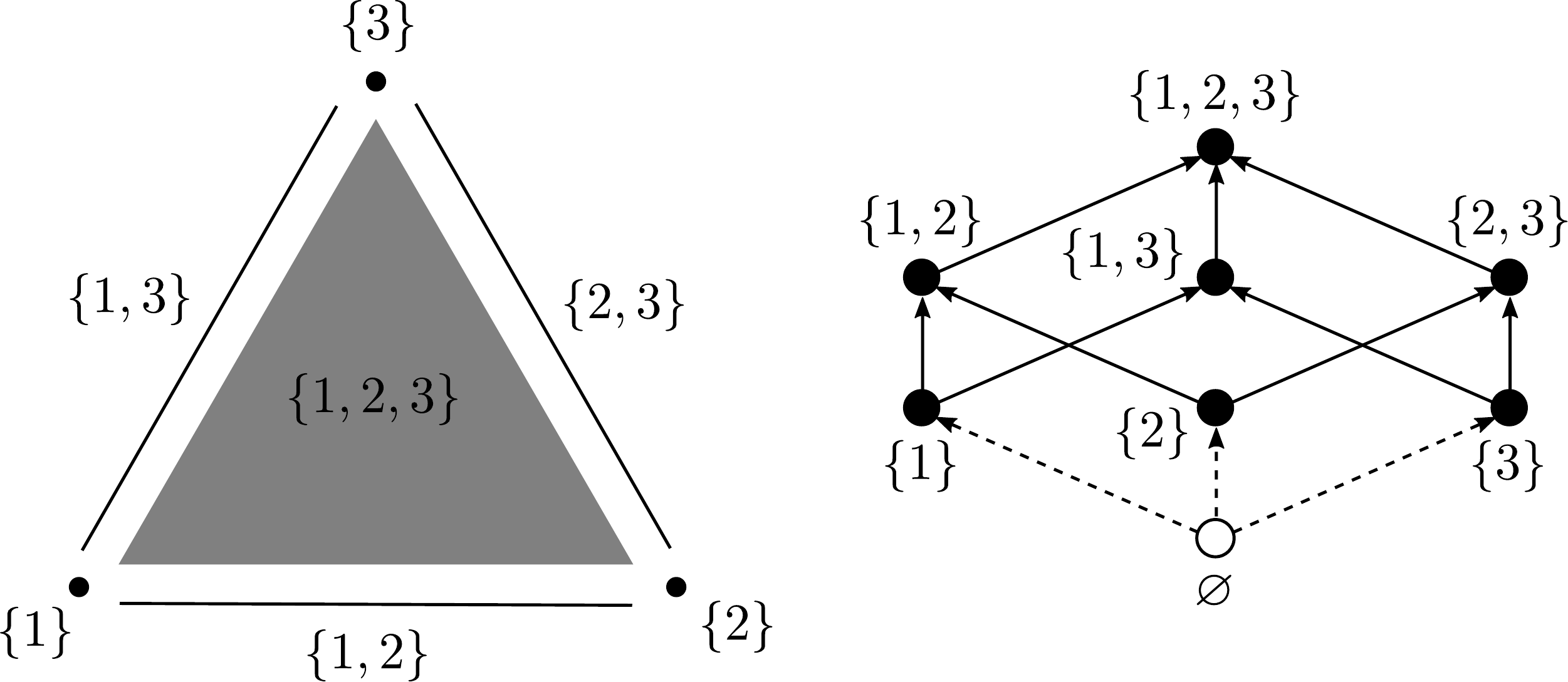}
\caption{Left: Decomposition of a simplex as the disjoint union of the relative interior of its faces. Right: Hasse diagram of the face lattice.\label{fig:simplex_decomposition}}
\label{default}
\end{center}
\end{figure}

\paragraph{Densities over the \textit{full} simplex.} 
In \S\ref{sec:densities}, we saw several distributions on the simplex $\triangle_{K-1}$. However, most of them (all but the hard concrete, HardKuma, and Gaussian sparsemax) assign zero probability to all faces but $\mathrm{ri}(\triangle_{K-1})$, that is, $\mathrm{Pr}\{\bm{y} \in f_\mathcal{I}\} = \int_{f_\mathcal{I}} p_Y(\bm{y}) = 0$ for any $\mathcal{I} \ne [K]$. In fact, any proper density (without Diracs) has this limitation: This is because the non-maximal faces (i.e., all faces except $\mathrm{ri}(\triangle_{K-1})$) have zero Lebesgue measure in $\mathbb{R}^{K-1}$. It is possible to circumvent this by defining densities that contain products of Dirac functions (as shown above for the case $K=2$). However, there is a more elegant construction that does not require generalized functions, as we shall see. 

The key trick is to replace the Lebesgue measure by a different measure that takes all simplex faces into account. 
\begin{definition}[Direct sum measure]\label{def:direct_sum}
We define the direct sum measure on $\triangle_{K-1}$ as
\begin{equation}\label{eq:direct_sum_measure}
\mu^\oplus(A) = \sum_{f \in \bar{\mathcal{F}}(\triangle_{K-1})} \mu_f(A \cap \mathrm{ri}(f)),
\end{equation}
where $\mu_f$ is the $\mathrm{dim}(f)$-dimensional Lebesgue measure for  $\mathrm{dim}(f)>0$, and the counting measure for $\mathrm{dim}(f) = 0$.
\end{definition}

We show in Appendix~\ref{sec:measure} that $\mu^{\oplus}$ is a valid measure on $\triangle_{K-1}$ under the product $\sigma$-algebra of its faces. 
We can then define a probability density $p_Y^\oplus(\bm{y})$ with respect to this measure and use it to compute probabilities of measurable subsets of $\triangle_{K-1}$. 
Such probability distribution can equivalently be defined as follows:
\begin{enumerate}
\item Define a probability mass function $P_F(f)$ on $\bar{\mathcal{F}}(\triangle_{K-1}) \simeq 2^{[K]} \setminus \{\varnothing\}$ (using the counting measure);
\item For each face $f \in \bar{\mathcal{F}}(\triangle_{K-1})$, define a probability density $p_Y(\bm{y} \mid f)$ over $\mathrm{ri}(f)$ (using the Lebesgue measure in $\mathbb{R}^{\mathrm{dim}(f)}$).
\end{enumerate}
Thus, random variables with a distribution of this form have a discrete part ($P_F(f)$) and a continuous part ($p_Y(\bm{y} \mid f)$); we call them \textbf{\concrete random variables}.%
\footnote{This direct sum measure and the two-step procedure above is related to the concept of ``manifold stratification'' and sampling procedures proposed in the statistical physics literature \citep{holmes2020simulating}. A similar formulation has also been recently considered (for a few special cases) by \citet{murady2020probabilistic} and \citet{vanderwel2020improving}.} %
This is formalized in the following definition. 

\begin{definition}[\concrete random variable] \label{def:concrete_rv}
Let $Y$ denote a  random variable over points in the simplex (including the boundary) and $F$ a discrete random variable over faces. Since the mapping from $Y$ to its face $F$ is deterministic, we have $p^\oplus_Y(\bm{y}) = P_F(f) p_{Y\mid F}(\bm{y} \mid f)$ for $\bm{y} \in \mathrm{ri}(f)$. 
The probability of a set $A \subseteq \triangle_{K-1}$ is given by:
\begin{equation}\label{eq:prob_concrete}
\mathrm{Pr}\{\bm{y} \in A\} = \int_{A} p_Y^\oplus(\bm{y}) d\mu^{\oplus} = \sum_{f \in \bar{\mathcal{F}}(\triangle_{K-1})} P_F(f) \int_{A_f} p_{Y\mid F}(\bm{y} \mid f),
\end{equation}
where $A_f = A \cap \mathrm{ri}(F)$.
\end{definition}
The expression \eqref{eq:prob_concrete} may be regarded as a manifestation of the law of total probability mixing discrete and continuous variables. 
Using this expression, we can for example write the expectation of a \concrete random variable as 
\begin{equation}
\mathbb{E}_{p_Y^\oplus}[Y] = \mathbb{E}_{P_F}\left[\mathbb{E}_{p_{Y\mid F}}[Y\mid F=f]\right].
\end{equation}

Note that both discrete and continuous distributions are recovered with our definition: If $P_F(f) = 0$ for $\mathrm{dim}(f) > 0$, we have a discrete categorical distribution, which only assigns probability to the 0-faces, \textit{i.e.}, the vertices. On the other extreme, if  $P_F(f) = 1$ for $\mathrm{dim}(f) = K-1$ (that is, if $f = \triangle_{K-1}$), and $0$ otherwise, we have a continuous distribution confined to $\mathrm{ri}(\triangle_{K-1})$. 
That is, \textbf{\concrete random variables include purely discrete and purely continuous random variables as particular cases.}
Of course, to parametrize distributions of high-dimensional mixed random variables, it is not efficient to consider all degrees of freedom suggested in Definition~\ref{def:concrete_rv}, since there are $2^{K}-1$ many faces, excluding the empty set. Instead, it seems a reasonable idea to derive parametrizations that exploit the lattice structure of the faces.

\paragraph{Example: Gaussian sparsemax.} Let us revisit the 2D Gaussian sparsemax example from \eqref{eq:gaussian_sparsemax}. For $K=2$, using Dirac deltas, the density with respect to the Lebesgue measure in $\mathbb{R}$ has the form
\begin{equation}
p_Y(y) = \left\{
\begin{array}{ll}
P_0\delta_{0}(y) + P_1\delta_{1}(y) + \mathcal{N}(y; z, \sigma^2), & \text{if $y \in [0,1]$} \\
0, & \text{otherwise,}
\end{array}
\right.
\end{equation}
with $P_0 = \frac{1-\mathrm{erf}(z/(\sqrt{2}\sigma))}{2}$ and $P_1 = \frac{1+\mathrm{erf}((z-1)/(\sqrt{2}\sigma))}{2}$. 
The same distribution can be expressed via the density $p_Y^\oplus(y) = P_F(f) p_{Y\mid F}(\bm{y}\mid f)$ as 
\begin{gather}
P_F(\{0\}) = P_0, \quad P_F(\{1\}) = P_1, \quad P_F([0,1]) = 1-P_0-P_1, \nonumber\\
p_{Y\mid F}(y \mid F=[0,1])(y) \,\,=\,\, \frac{\mathcal{N}(y; z, \sigma^2)}{1 - P_0 - P_1}.
\end{gather}
For $K>2$, obtaining closed form expressions for $P_F$ and $p_{Y\mid F}$ seems considerably harder. It is possible to compute the face probability mass function $P_F$ by evaluating an integral using distributions of order statistics of the Gaussian \citep{vieira2021order,vieira2021smallest}.



\section{Information Theory for \Concrete Random Variables}\label{sec:information_theory}

Now that we have the tools to define probability distributions that take into account all the faces of the simplex without the need for Dirac deltas, we  proceed to defining the \textbf{entropy} of such distributions. 

The entropy of a random variable $X$ with respect to a measure $\mu$ is: 
\begin{equation}\label{eq:entropy}
H^{\mu}(X) = -\int_{\mathcal{X}} p_X(x) \log p_X(x) d\mu(x),
\end{equation}
where $p_X(x)$ is a probability density satisfying $\int_{\mathcal{X}} p_X(x) d\mu(x) = 1$. 
When $\mathcal{X}$ is finite and $\mu$ is the counting measure, the integral becomes a sum and we recover \textbf{Shannon's discrete entropy}, which is non-negative and upper bounded by $\log |\mathcal{X}|$, the entropy of the uniform distribution. 
When $\mathcal{X} \subseteq \mathbb{R}^k$ is continuous and $\mu$ is the Lebesgue measure, we recover the \textbf{differential entropy}, which can be negative and, for compact $\mathcal{X}$, is upper bounded by the logarithm of the volume of $\mathcal{X}$. The maximal value corresponds to a continuous uniform distribution on $\mathcal{X}$. 

For example, the differential entropy of a Dirichlet random variable $Y \sim \mathrm{Dir}(\bm{\alpha})$ (cf.~\eqref{eq:dirichlet}) is 
\begin{equation}
H(Y) = \log B(\bm{\alpha}) + (\alpha_0 - K) \psi(\alpha_0)  - \sum_{k=1}^K (\alpha_k-1)\psi(\alpha_k),
\end{equation}
where $\alpha_0 = \sum_{k}^K \alpha_k$ and $\psi$ is the digamma function. 
When $\bm{\alpha} = \mathbf{1}$, this becomes a flat (uniform) density and the entropy attains its maximum value:
\begin{equation}\label{eq:entropy_flat_dirichlet}
H(Y) = \log B(\bm{\alpha}) = -\log (K-1)!.
\end{equation}
This value is negative for $K>2$; it follows that the differential entropy of any distribution in the simplex is negative.

\paragraph{Direct sum entropy.} 
What happens if we plug in \eqref{eq:entropy} the direct sum measure \eqref{eq:direct_sum_measure}? 
Since $F$ depends deterministically on $Y$, we have $H(Y, F) = H(Y)$. 
We therefore define the direct sum entropy of a \concrete random variable as follows. 
\begin{definition}[Direct sum entropy]\label{def:direct_sum_entropy}
Let $Y$ be a \concrete random variable. The direct sum entropy of $Y$ is 
\begin{small}
\begin{align}\label{eq:entropy_simplex}
H^{\oplus}(Y) &:=  H(F) + H(Y \mid F) \\
&= \underbrace{-\!\!\sum_{f \in \bar{\mathcal{F}}(\triangle_{K-1})} \!\! P_F(f) \log P_F(f)}_{\text{discrete entropy}} 
+ \sum_{f \in \bar{\mathcal{F}}(\triangle_{K-1})} \!\! P_F(f) \underbrace{\left(-\int_{f} p_{Y\mid F}(\bm{y}\mid f) \log p_{Y\mid F}(\bm{y}\mid f)\right)}_{\text{differential entropy}}.\nonumber
\end{align}
\end{small}
\end{definition}
As shown in Definition~\ref{def:direct_sum_entropy}, the direct sum entropy has two components: a \textbf{discrete entropy over faces} and an \textbf{expectation of a differential entropy over each face.} 

\paragraph{Relation to optimal codes.} 
The discrete entropy of a random variable representing an alphabet symbol corresponds to the average length of the optimal code for the symbols in the alphabet, in a lossless compression setting. Besides, it is known \citep{cover2012elements} that the optimal number of bits to encode a $D$-dimensional continuous random variable with $N$ bits of precision equals its differential entropy (in bits) plus $ND$.%
\footnote{\citet[Theorem~9.3.1]{cover2012elements} provide an informal proof for $D=1$, but it is straightforward to extend the same argument for $D>1$.} %
Therefore, \textbf{the direct sum entropy \eqref{eq:entropy_simplex} has an interpretation in terms of code optimality: it is the average length of the optimal code where the sparsity pattern of $\bm{y} \in \triangle_{K-1}$ must be encoded losslessly and where there is a predefined bit precision for the fractional entries of $\bm{y}$.} This is formalized as the following
\begin{proposition}
Let $Y \sim p^\oplus_Y(\bm{y})$ be a \concrete random variable. 
In order to encode the face of $Y$ losslessly and to ensure a $N$-bit precise value of a point in that face we need the following number of bits: %
\begin{equation}\label{eq:coding_entropy}
H^{\oplus}_N(Y) = H^{\oplus}(Y) + N\sum_{k=1}^K (k-1) \sum_{f \in \mathcal{F}(\triangle_{K-1}) \atop \mathrm{dim}(f) = k-1} P_F(f).
\end{equation}
\end{proposition}

\paragraph{Example: Gaussian sparsemax.} 
Revisiting our Gaussian sparsemax running example with $K=2$, we have 
\begin{equation}
p_Y(y) = P_0\delta_{0}(y) + P_1\delta_{1}(y) + \mathcal{N}(y; z, \sigma^2),
\end{equation}
with $P_0 + P_1 + \int_{0}^1 \mathcal{N}(y; z, \sigma^2) dy = 1$.
The direct sum entropy becomes
    \begin{align}
    H(Y) &= H(F) + H(Y\mid F)\nonumber\\
    &= H([P_0, P_1, 1-P_0-P_1]) - (1-P_0-P_1) \int_{0}^{1} \frac{\mathcal{N}(y; z, \sigma^2) }{1-P_0-P_1} \log \frac{\mathcal{N}(y; z, \sigma^2) }{1-P_0-P_1} dy\nonumber\\
    &= \underbrace{-P_0 \log P_0 - P_1 \log P_1}_{\text{discrete part}} \underbrace{- \int_{0}^{1} \mathcal{N}(y; z, \sigma^2) \log \mathcal{N}(y; z, \sigma^2)  dy}_{\text{continuous part}}.
    \end{align}
For the Gaussian sparsemax \eqref{eq:gaussian_sparsemax}, we have $P_0 = \frac{1-\mathrm{erf}(z/(\sqrt{2}\sigma))}{2}$, $P_1 = \frac{1+\mathrm{erf}((z-1)/(\sqrt{2}\sigma))}{2}$, and 
\begin{align}
\lefteqn{-\int_{0}^{1} \mathcal{N}(y; z, \sigma^2) \log \mathcal{N}(y; z, \sigma^2)  dy=}\nonumber\\ &= \int_{0}^{1} \mathcal{N}(y; z, \sigma^2) \left(\log(\sqrt{2\pi\sigma^2}) + \frac{(y-z)^2}{2\sigma^2}\right)dy \nonumber\\
&= (1-P_0-P_1) \log(\sqrt{2\pi\sigma^2}) + \frac{\sigma}{2} \int_{\frac{-z}{\sigma}}^{\frac{1-z}{\sigma}} \mathcal{N}(t; 0, 1) t^2 dt \nonumber\\
&= (1-P_0-P_1) \log(\sqrt{2\pi\sigma^2}) \nonumber\\ 
& \quad+ \frac{\sigma}{2} \left( \frac{\mathrm{erf}\left(\frac{1-z}{\sqrt{2\sigma^2}}\right) - \mathrm{erf}\left(-\frac{z}{\sqrt{2\sigma^2}}\right)}{2} - \frac{1-z}{\sigma}\mathcal{N}\left(\frac{1-z}{\sigma}; 0, 1\right) - \frac{z}{\sigma}\mathcal{N}\left(-\frac{z}{\sigma}; 0, 1\right)\right),
\end{align}
which leads to a closed form for the entropy.


\begin{figure}[t]
\begin{center}
\includegraphics[width=0.7\columnwidth]{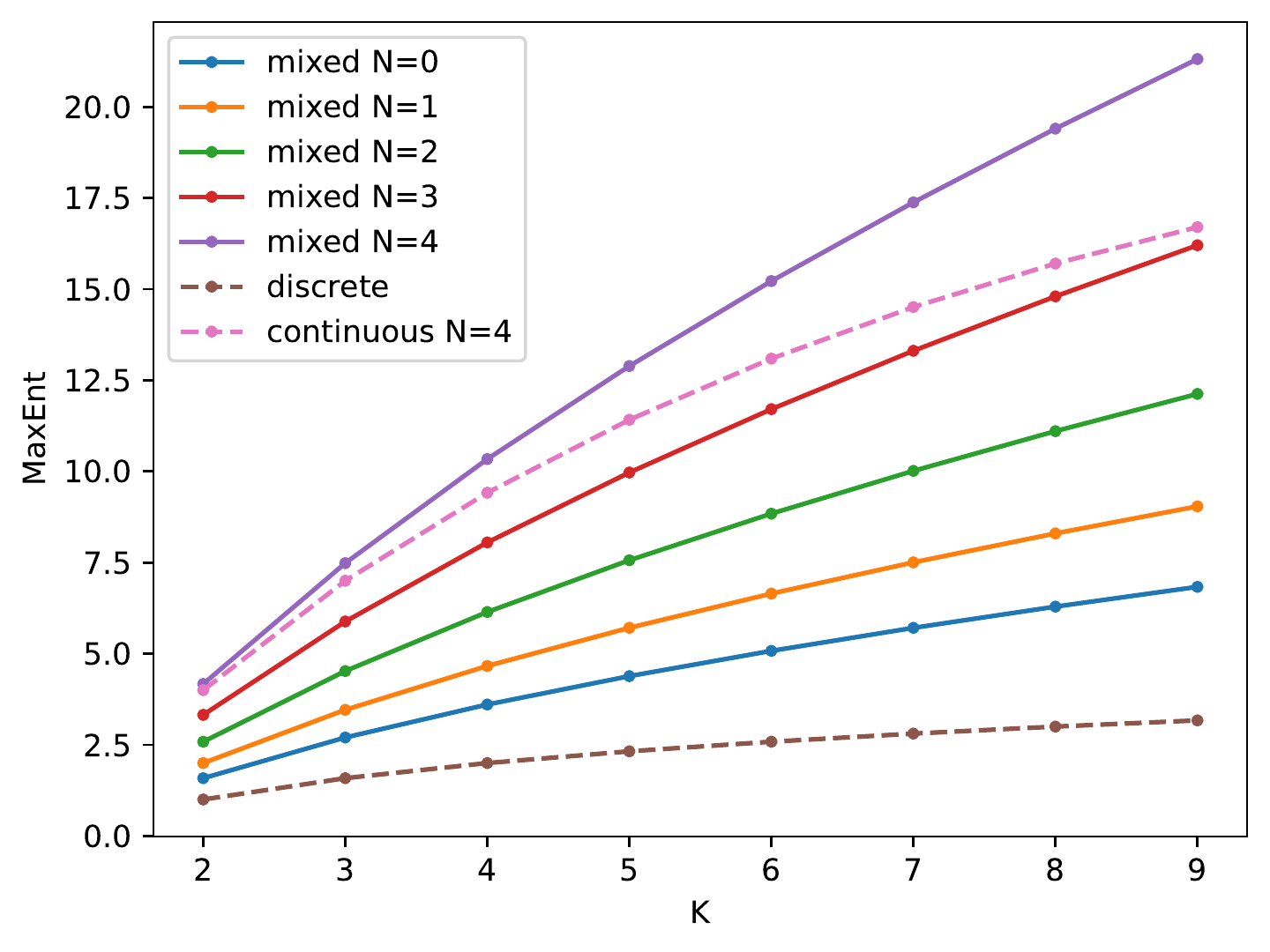}
\caption{Maximum entropies for \concrete distributions for several values of bit precision $N$, as a function of the simplex dimensionality $K-1$. Shown are also the maximum entropies for the corresponding discrete and continuous cases, for comparison.\label{fig:maxent}}
\label{default}
\end{center}
\end{figure}

\paragraph{Maximum entropy density in the full simplex.} 
An important question is what is the distribution $p^\oplus_Y(\bm{y})$ with the largest entropy in the \emph{full} simplex. Considering only the maximal face, which corresponds to $\mathrm{ri}(\triangle_{K-1})$, this is the flat distribution, whose entropy is given in \eqref{eq:entropy_flat_dirichlet}. In our definition of entropy in \eqref{eq:entropy_simplex} this corresponds to a deterministic $F$ which puts all probability mass in this maximal face. But constraining ourselves to a single face is quite a limitation, and in particular \emph{knowing} this constraint provides valuable information that intuitively should reduce entropy!%
\footnote{At the opposite extreme, if we only assign probability to pure vertices, \textit{i.e.}, if we constrain to \emph{minimal} faces, a uniform choice leads to a (Shannon) entropy of $\log K$. We will see in the sequel that looking at \emph{all} faces further increases entropy.} %
What if we consider densities that assign probability to the boundaries? 
Looking at \eqref{eq:entropy_simplex}, we see that the differential entropy term $H(Y \mid F=f)$ can be maximized separately for each $f$, the solution being the flat distribution on face $f \simeq \triangle_{k-1}$, which has entropy $-\log (k-1)!$, where $1 \le k \le K$. By symmetry, all faces of the same dimension $k-1$ look the same, and there are ${K \choose k}$ of them. Therefore the maximal entropy distribution is attained with $P_F$ of the form $P_F(f) = g(k) / {K \choose k}$ where $g : [K] \rightarrow \mathbb{R}_+$ is a function satisfying $\sum_{k=1}^K g(k) = 1$ (which can be regarded as a categorical probability mass function). 
If we choose a precision of $N$ bits, this leads to:
\begin{align}
H_N^\oplus(Y) &= -\sum_{f \in \bar{\mathcal{F}}(\triangle_{K-1})} P_F(f) \log P_F(f) 
+ \sum_{f \in \bar{\mathcal{F}}(\triangle_{K-1})} P_F(f) \left(-\int_{f} p_{Y\mid F}(\bm{y}\mid f) \log p_{Y\mid F}(\bm{y}\mid f)\right)\nonumber\\
& \quad + N\sum_{k=1}^K (k-1)g(k) \\
&= -\sum_{k=1}^K g(k) \log \frac{g(k)}{{K \choose k}}
- \sum_{k=1}^K g(k) (\log (k-1)! - N(k-1)) \nonumber\\
&= -\sum_{k=1}^K g(k) \log g(k) + 
\sum_{k=1}^K g(k) \log \frac{{K \choose k} 2^{N(k-1)}}{(k-1)!}.
\end{align}
This is a entropy-regularized argmax problem, hence the $g(k)$ that maximizes this objective is the softmax transformation of the vector with components $\log \frac{{K \choose k}2^{N(k-1)}}{(k-1)!}$, that is, 
\begin{equation}
g(k) = \frac{\frac{{K \choose k}2^{N(k-1)}}{(k-1)!}}{\sum_{j=1}^K \frac{{K \choose j}2^{N(j-1)}}{(j-1)!}}, 
\end{equation}
and the maximum entropy value is %
\begin{equation}
H^\oplus_{N, \mathrm{max}}(Y) = \log {\sum_{k=1}^K \frac{{K \choose k}2^{N(k-1)}}{(k-1)!}} = \log L_{K-1}^{(1)}(-2^N).
\end{equation}
where $L_{n}^{(\alpha)}(x)$ denotes the \textbf{generalized Laguerre polynomial} \citep{sonine1880recherches}. 
For $K=2$ this value is $\log(2 + 2^N)$, for $K=3$ it is $\log(3 + 3\cdot 2^N + 2^{2N-1})$, etc. 
A plot is shown in Figure~\ref{fig:maxent}.

For example, for $K=2$, we obtain $g(1) = \frac{2}{2 + 2^N}$,  $g(2) = \frac{2^N}{2 + 2^N}$, and $H^\oplus_N(Y) = \log(2 + 2^N)$, therefore, in the worst case, we need at most $\log_2 (2 + 2^N)$ bits to encode $Y \in \triangle_1$ with bit precision $N$. 
This is intuitive: the faces of $\triangle_1$ are the two vertices $\{(0,1)\}$ and $\{(1,0)\}$ and the line segment $[0,1]$. The first two faces have a probability of $\frac{1}{2+2^N}$ and the last one have a probability $\frac{2^N}{2 + 2^N}$. To encode a point in the simplex we first need to indicate which of these three faces it belongs to (which requires $\frac{2}{2+2^N} \log_2 (2+2^N) + \frac{2^N}{2 + 2^N} \log_2 \frac{2 + 2^N}{2^N} = \log_2 (2+2^N) - \frac{N2^N}{2+2^N}$ bits), and with $\frac{2^N}{2 + 2^N}$ probability we need to encode a point uniformly distributed in the segment $[0,1]$ with $N$ bit precision, which requires extra $\frac{N2^N}{2 + 2^N}$ bits on average. Putting this all together, the total number of bits is $\log_2 (2+2^N)$, as expected. 



\paragraph{KL divergence and mutual information.}
Since the direct sum entropy $H^\oplus(Y)$ in Definition~\ref{def:direct_sum_entropy} equals a ``classical'' joint entropy $H(Y,F)$, extensions to Kullback-Leibler (KL) divergence and mutual information are straightforward and they are both non-negative. 

\begin{definition}[KL divergence]
The KL divergence between distributions $p_Y^\oplus \equiv (P_F, p_{Y\mid F})$ and $q_Y^\oplus \equiv (Q_F, q_{Y\mid F})$ is:
\begin{small}
\begin{align}\label{eq:kl_simplex}
\lefteqn{KL^\oplus(p_Y^\oplus\|q_Y^\oplus):=}\nonumber\\ 
&\quad := KL(P_F\|Q_F) + \mathbb{E}_{f \sim P_F} \left[ KL(p_{Y\mid F}(\cdot \mid F=f) \| q_{Y\mid F}(\cdot \mid F=f) \right] \\
&\quad = \underbrace{-\sum_{f \in \bar{\mathcal{F}}(\triangle_{K-1})} P_F(f) \log \frac{P_F(f)}{Q_F(f)}}_{\text{discrete KL}} 
+ \sum_{f \in \bar{\mathcal{F}}(\triangle_{K-1})} P_F(f) \underbrace{\left(-\int_{f} p_{Y\mid F}(\bm{y}\mid f) \log \frac{p_{Y\mid F}(\bm{y}\mid f)}{q_{Y\mid F}(\bm{y}\mid f}\right)}_{\text{continuous KL}}.\nonumber
\end{align}
\end{small}
\end{definition}
Intuitively, the KL divergence between $p_Y^\oplus$ and $q_Y^\oplus$ expresses the additional average code length if we encode variable $Y \sim p_Y^\oplus(\bm{y})$ with a code that is optimal for distribution $q_Y^\oplus(\bm{y})$. 

Note that the KL divergence becomes $+\infty$ if $\mathrm{supp}(P_F) \nsubseteq \mathrm{supp}(Q_F)$ (\textit{i.e.}, if  $p_Y$ assigns non-zero probability to a face which has zero probability under $q_Y$ -- in terms of the face lattice in Figure~\ref{fig:simplex_decomposition}, for the KL divergence to be finite, there must be a path in the diagram from the face support of $p_Y$ to the face support of $q_Y$) or if there is some face where $\mathrm{supp}(p_{Y\mid F}(\cdot \mid F=f)) \nsubseteq \mathrm{supp}(q_{Y\mid F}(\cdot \mid F=f))$.%
\footnote{In particular, this means that \concrete distributions shall not be used as a relaxation in VAEs with purely discrete priors using the ELBO -- rather, the prior should be also \concrete.}

\begin{definition}[Mutual information]
For \concrete random variables $Y$ and $Z$, the mutual information between $Y$ and $Z$ is 
\begin{align}
I^\oplus(Y; Z) &= H^\oplus(Y) - H^\oplus(Y\mid Z) \nonumber\\
&= H(F) + H(Y\mid F) - H(F\mid Z) + H(Y \mid F, Z)\nonumber\\ 
&= I(F; Z) + I(Y; Z \mid F) \ge 0. 
\end{align}
\end{definition}
With these ingredients it is possible to provide counterparts for channel coding theorems by combining Shannon's discrete and continuous channel theorems.

\section{\Concrete Languages}\label{sec:concrete_languages}

%

So far, we looked as probability distributions associated to a single symbol occurrence. However, communication requires the generation of multiple symbols, better expressed as \textit{strings}. 
In this section, we will explore connections between the framework developed so far and formal languages. 
There are two ways through which sparse distributions over symbols could be applicable to strings:
\begin{itemize}
\item We could use it to build a weighted lattice or search tree whose weights correspond to the symbol probabilities. This has been done with the entmax transformation and autoregressive models for sequence-to-sequence prediction by \citet{peters2019sparse}. An interesting direction would be to introduce a \textit{merge} operation to such models to prevent the search tree to grow exponentially.
\item Another way, which we will present in this section, is to consider sequences of \concrete symbols and the languages formed by such sequences.
\end{itemize}
We start by quickly reviewing weighted finite state automata on discrete (finite) alphabets, and then generalize them to the \concrete case.

\subsection{Weighted Finite State Automata}\label{sec:wfsa}

Let $\Sigma=[K]$ denote our alphabet. The Kleene closure of $\Sigma$, denoted $\Sigma^\star$, is the set of all finite strings in the alphabet including the empty string $\epsilon$, $\Sigma^\star := \{\epsilon\} \cup \Sigma \cup \Sigma^2 \cup \ldots$. 
We denote by $\triangle(\Sigma) := \triangle_{K-1}$ the set of distributions over symbols in $\Sigma$.

Let $(\mathbb{K}, \oplus, \otimes, \bar{0}, \bar{1})$ be a semiring. A \textbf{weighted finite-state automaton} (WFSA) is a tuple $(\Sigma, S, I, F, \delta, \lambda, \rho)$, where $\Sigma$ is the alphabet, $S$ is a finite non-empty set of states, $I \subseteq S$ is the set of initial states, $F \subseteq S$ is the set of final states, $\delta:S\times (\Sigma \cup \{\epsilon\}) \times S \rightarrow \mathbb{K}$ is the transition function, $\lambda: I \rightarrow \mathbb{K}$ is the initial weight function, and $\rho: F \rightarrow \mathbb{K}$ is the final weight function. 

A WFSA can be represented as a directed, labeled, and weighted graph, where the vertices are the states $S$ and the labeled and weighted edges are all pairs $(s,t, a, w) \in S\times S \times (\Sigma \cup \{\epsilon\}) \times \mathbb{K}$ such that $\delta(s,a,t) = w$ (where edges with weight $\bar{0}$ can be omitted). Vertices corresponding to initial and final states are decorated with the corresponding values of the functions $\lambda$ and $\rho$. See Figure~\ref{fig:wfsa}  for an illustration. 

The weight of a path from the initial state to a final state is the product $\otimes$ of weights of all the edges in that path times the weight of the initial and final states. The weight of a string $x \in \Sigma^\star$ according to an WFSA $\mathcal{A}$, denoted $[[\mathcal{A}]](x)$, is the sum $\oplus$ of the weights of all paths consistent with $x$.

A WFSA is said to be \textbf{deterministic} if it has a unique initial state, no epsilon-transitions, and if no two transitions leaving any state share the same input label. A WFSA over the Boolean semiring $(\mathbb{K}, \oplus, \otimes) = (\{0,1\}, \vee, \wedge)$ is simply called a \textbf{finite-state automaton} (FSA). 
In the sequel, whenever we are not talking about an FSA, we will always assume that $(\mathbb{K}, \oplus, \otimes) = (\mathbb{R}_+, +, \times)$, called the \textit{probability semiring} \citep{mohri2004weighted}.  

Given a string $x \in \Sigma^\star$, we say that an FSA $\mathcal{A}$ \textit{accepts} $x$ if $[[\mathcal{A}]] = \bar{1}$. 
A \textbf{language} $\mathcal{L}$ is a set of strings, $\mathcal{L} \subseteq \Sigma^\star$. 
We say that an FSA \textit{recognizes} $\mathcal{L}$ if it accepts all the strings in $\mathcal{L}$ and no other strings. 
Any FSA can be ``determinized,'' (\textit{i.e.}, transformed into a equivalent deterministic FSA that recognizes the same language). This determinization can be achieved through the \textbf{powerset construction} \citep{rabin1959finite}, which in the worst case increases  the number of states exponentially. This establishes an equivalence between the two automata. 
On the other hand, WFSAs may not be determinizable (\textit{i.e.}, turned into an equivalent deterministic WFSA that assigns every string the same weight), unless additional properties are satisfied \citep{mohri2004weighted}. 
A language $\mathcal{L}$ is called \textbf{regular} if there is an FSA that recognizes $L$. Regular languages have many closure properties \citep{hopcroft2001introduction}: the union, intersection, negation, and concatenation of regular languages is also a regular language.

\begin{figure}[t]
\begin{center}
\includegraphics[width=1\columnwidth]{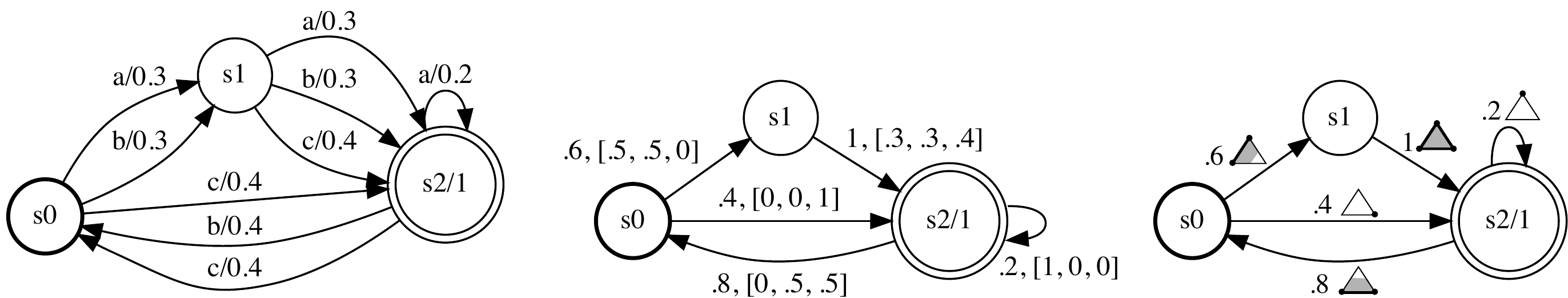}
\caption{Two representations of the same WFSA, and a representation of a MWFSA. Left: Typical representation of an WFSA with an arrow for each symbol connecting two states. Middle: A more compact representation represented as a weighted graph along with a vector of emission probabilities for each arrow. Right: Representation of a MWFSA where emission probabilities are replaced by densities in the simplex.\label{fig:wfsa}}
\label{default}
\end{center}
\end{figure}

\paragraph{Decomposition of WFSA transition weights using probabilities.} 
To simplify, we will assume that the WFSA has no epsilon transitions (they can always be removed if needed). 
We can decompose the transition function of a WFSA as $\delta(s,a,t) = w(s,t) P(a \mid s,t)$, where $w(s,t) \in \mathbb{R}_+$ and $P(\cdot \mid s,t)$ is a probability mass function over $a \in \Sigma$. 
If the WFSA is stochastic -- an equivalent stochastic WFSA can always be obtained by the weight push algorithm \citep{mohri2004weighted} -- then we can further write $w(s,t)$ in the form $w(s,t) = P(t \mid s)$ for a probability distribution $P(t \mid s)$, in which case the transition function can be seen as a product of \textbf{transition} and \textbf{emission} probabilities, as a \textbf{hidden Markov model} \citep{rabiner1986introduction}.
We denote by $\bm{y}(s,t) := P(\cdot \mid s,t) \in \triangle(\Sigma)$ the vector of emission probabilities associated with the transition between a state $s$ and a state $t$. 
For a deterministic WFSA, we have a disjoint union $\Sigma = \bigsqcup_{t \in S} \mathrm{supp}(\bm{y}(s,t))$ for every state $s \in S$. 
If the WFSA is not deterministic, we have the weaker relation $\Sigma \supseteq \bigcup_{t \in S} \mathrm{supp}(\bm{y}(s,t))$. 
This way, we can represent a WFSA alternatively as a directed graph where the vertices are the states $S$, but with fewer edges: the edges are all pairs $(s,t) \in S\times S$ such that $w(s,t) > 0$. To each edge, we associate a point in the simplex, $\bm{y}(s,t) \in \triangle(\Sigma)$, which represents the probability mass function $P(a \mid s,t)$. This is shown in Figure~\ref{fig:wfsa}.%
\footnote{For an FSA, this construction can be recovered if we define $\bm{y}(s,t)$ as a uniform distribution over its support, and $w(s,t) = |\mathrm{supp}(\bm{y}(s,t))|$, while deleting edges for which there are no transitions.} %

\subsection{\Concrete Strings and Languages}

In this section, we extend symbolic languages and classes of languages to our mixed discrete-continuous space. We will define ``\concrete strings'' as sequences of symbols which do not need to be \textit{pure} -- they can be a (sparse) \textit{mixture} of multiple symbols. From this concept, we will go on to define \concrete languages and a class of \concrete regular languages, for which we derive some properties. 


We start by extending the definitions of \S\ref{sec:wfsa} as follows.

\begin{definition}[\Concrete strings and languages]

Let $\triangle(\Sigma) := \triangle_{K-1}$ denote the set of distributions over symbols in $\Sigma$. 
\begin{itemize}
\item The Kleene closure of $\triangle(\Sigma)$ is $\triangle(\Sigma)^\star := \{\epsilon\} \cup \triangle(\Sigma) \cup \triangle(\Sigma)^2 \cup \ldots$. 
\item An element of $\triangle(\Sigma)^\star$ is called a \textbf{\concrete string}. 
\item A \textbf{\concrete language} $\mathcal{L}$ is a set of \concrete strings, $\mathcal{L} \subseteq \triangle(\Sigma)^\star$. 
\item The \textbf{union} of \concrete languages $\mathcal{L}_1$ and  $\mathcal{L}_2$ is the set of \concrete strings that belong to either of the languages; their \textbf{intersection} is the set of \concrete strings that belong to both languages; the \textbf{concatenation} of  $\mathcal{L}_1$ and  $\mathcal{L}_2$ is set $\mathcal{L}_1 + \mathcal{L}_2 := \{x_1 x_2 \mid x_1 \in \mathcal{L}_1 \wedge x_2 \in \mathcal{L}_2\}$. 
The \textbf{negation} of a language $\mathcal{L}$ is the language $\triangle(\Sigma)^\star \setminus \mathcal{L}$.
\item The \textbf{skeleton} of a \concrete string $x \in \triangle(\Sigma)^\star$ is the (symbolic) string $\mathrm{skel}(x) \in (2^\Sigma \setminus \{\varnothing\})^\star$ of the same length with $(\mathrm{skel}(x))_i := \mathrm{supp}(x_i)$. 
\item The skeleton of a \concrete language is the set $\mathrm{skel}(\mathcal{L}) := \{ \mathrm{skel}(x) \mid x \in \triangle(\Sigma)^\star \}$. 
This is a (symbolic) language over the alphabet $2^\Sigma \setminus \{\varnothing\}$, \text{i.e.}, $\mathrm{skel}(\mathcal{L}) \subseteq (2^\Sigma \setminus \{\varnothing\})^\star$. 
\item A \textbf{projection} of a \concrete string $x \in \triangle(\Sigma)^\star$ is a (symbolic) string $u \in \Sigma^\star$ of the same length such that $u_i \in \mathrm{supp}(x_i)$ for each position $i$; we denote this by $u \in \mathrm{proj}(x)$. 
\item The projection of a \concrete language $\mathcal{L}$ is the discrete language in the same alphabet formed by all projections of \concrete strings in $\mathcal{L}$, $\mathrm{proj}(\mathcal{L}) := \{ u \in \Sigma^\star \mid u \in \mathrm{proj}(x), \,\, x \in \triangle(\Sigma)^\star \}$. 
\end{itemize}
\end{definition}

Intuitively, a \concrete language is a language made of ``symbols'' which do not need to be pure -- they can be a mixture of one or more symbols weighted by a probability. The skeleton of a \concrete language, on the other hand, ignores the weights but retains the subsets of symbols that are mixed; therefore, it can be seen as a language over the powerset vocabulary (removing the empty set) $2^\Sigma \setminus \{\varnothing\}$.


For example, 
$x = \texttt{a}\texttt{b}\left\{{_{\texttt{b:0.8}}^{\texttt{a:0.2}}}\right\}\texttt{a}$
is a \concrete string over $\Sigma=\{\texttt{a}, \texttt{b}\}$, where the two first and last symbols are pure, and the third symbol is the point $[0.2, 0.8]$ in the simplex $\triangle_1$, which we can interpret as a mixture of $a$ and $b$. 
The only two projections of $x$ are $\text{\tt abaa}$ and $\text{\tt abba}$. The skeleton of $x$ is the string $\texttt{a}\texttt{b}{_{\texttt{b}}^{\texttt{a}}}\texttt{a}$, over $2^\Sigma\setminus \{\varnothing\} = \{\texttt{a}, \texttt{b},_{\texttt{b}}^{\texttt{a}}\}$.

\paragraph{\Concrete WFSA.} 
The next step is to define
\begin{definition}[\Concrete WFSA]
A \textbf{\concrete weighted finite-state automaton} (MWFSA) over the probability or Boolean semiring $(\mathbb{K}, \oplus, \otimes, \bar{0}, \bar{1})$ is a tuple $(\Sigma, S, I, F, \delta, \lambda, \rho)$ 
where $\Sigma$ is the alphabet, $S$ is a finite non-empty set of states, $I \subseteq S$ is the set of initial states, $F \subseteq S$ is the set of final states, $\delta:S\times (\triangle(\Sigma)\cup \triangle(\{\epsilon\}) ) \times S \rightarrow \mathbb{K}$ is the transition function, $\lambda: I \rightarrow \mathbb{K}$ is the initial weight function, and $\rho: F \rightarrow \mathbb{K}$ is the final weight function. 
An MWFSA over the Boolean semiring is simply called a \textbf{\concrete finite-state automaton} (MFSA). 
\end{definition}
That is, an MWFSA is similar to a WFSA, except that the transition function is $\delta:S\times (\triangle(\Sigma)\cup \{\epsilon\} ) \times S \rightarrow \mathbb{K}$. 
That is, instead of symbols of a finite alphabet, $a \in \Sigma$, each transition in a MWFSA is dictated by a point in the simplex, $\bm{y} \in \triangle(\Sigma)$. 

An MWFSA cannot be visualized as a directed labeled weighted graph as we did initially for the WFSA, since there would be uncountably many labels. 
Instead, we use a \textbf{decomposition of transition weights} similarly to what we did for the WFSA, but now using densities instead of probability mass functions. That is, we decompose the transition functions as $\delta(s,\bm{y},t) = w(s,t) p_Y(\bm{y} \mid s,t)$, where $w(s,t) \in \mathbb{R}_+$ and $p_Y(\cdot \mid s,t)$ is a probability density function over $\bm{y} \in \triangle(\Sigma)$. 
We require that $\mathrm{supp}(p_Y(\cdot \mid s,t))$ is a \textbf{measurable set} (\textit{i.e.} it belongs to the $\sigma$-algebra of the direct sum measure in Definition~\ref{def:direct_sum}; see Appendix~\ref{sec:measure} for details). 
If the MWFSA is stochastic, then we further have $w(s,t) = P(t \mid s)$, in which case the transition function is a product of a \textbf{transition probability} and an \textbf{emission density}.
For a deterministic MWFSA, we have a disjoint union $\triangle(\Sigma) = \bigsqcup_{t \in S} \mathrm{supp}(p_Y(\cdot \mid s,t))$ for every state $s \in S$. 
If the MWFSA is not deterministic, we have the weaker relation $\triangle(\Sigma) \supseteq \bigcup_{t \in S} \mathrm{supp}(p_Y(\cdot \mid s,t))$. 
This way, we can represent a MWFSA as a directed graph where the vertices are the states $S$ and the edges are all pairs $(s,t) \in S\times S$ such that $w(s,t) > 0$. To each edge, we associate a density function, $p_Y(\cdot \mid s,t)$. This is shown in Figure~\ref{fig:wfsa}.%
\footnote{For a MFSA, this construction can be recovered if we define $p_Y(\cdot \mid s,t)$ as a uniform density over its support, and $w(s,t) = |\mathrm{supp}(p_Y(\cdot \mid s,t))|$, while deleting edges for which there are no transitions.} %

\paragraph{\Concrete regular languages and closure properties.} 
Let us assume an MFSA $\mathcal{A}$. 
A \concrete string $x = (\bm{y}_1, \ldots, \bm{y}_N) \in \triangle(\Sigma)^\star$ is accepted by $\mathcal{A}$ iff $[[\mathcal{A}]](x) = 1$. A MFSA recognizes a language $\mathcal{L}$ if it accepts all strings in $\mathcal{L}$ and no other string. 
We say that a \concrete language is \textbf{regular} if it is recognized by a MFSA.

\begin{proposition}\label{prop:regular}
We have the following:
\begin{enumerate}
\item Any regular language is also a \concrete regular language.
\item Any nondeterministic MFSA is equivalent to some deterministic MFSA.
\item The skeleton of a \concrete regular language over $\Sigma$ is a regular language over $2^\Sigma$.
\item The projection of a \concrete regular language over $\Sigma$ is a regular language over $\Sigma$.
\item \Concrete regular languages are close under union, intersection, negation, and concatenation.
\end{enumerate}
\end{proposition}
\begin{proof}[Proof sketch]
The key to this proof is to generalize the powerset construction of \citet{rabin1959finite} (which establishes the classic equivalence between deterministic and non-deterministic FSAs). 
To prove 3, note that the skeleton of a \concrete regular language associated to a MFSA can be obtained by deleting the weights in the transition function and relabeling each edge $\bm{y}$ to $\mathrm{supp}(\bm{y})$, which turns that MFSA into a non-deterministic one over $2^\Sigma$. From 2, this must be a regular language. A detailed proof is in Appendix~\ref{sec:regular}.%
\end{proof}




\section{Conclusion and Future Work}

We presented a mathematical framework for handling \concrete random variables, while are an hybrid between discrete and continuous. Key to our framework is the use of a direct sum measure as an alternative to the Lebesgue-Borel and the counting measures, which considers all faces of the simplex. Based on this we present generalizations of information theoretic concepts and regular languages for \concrete symbols. 

We believe the framework described here is only scratching the surface. For example, we are not studying efficient parametrizations of \concrete densities (beyond the already known reparametrization trick) and we are not addressing yet the structured case. 
However, the combinatorial characterization of (convex) polytopes in terms of their face lattice  \citep[Lecture~2, \S2.2]{ziegler1995lectures}, \citep[\S3.2]{grunbaum2003convex} goes beyond the probability simplex. 
This suggests applying this characterization to other transformations which return a sparse vector in a marginal polytope (therefore a point in a lower dimensional face), such as the SparseMAP \citep{niculae2018sparsemap}. 
Another interesting direction is to consider an infinite, countable simplex, which would enable a non-parametric usage of this framework, where the number of faces with nonzero probability can grow unbounded. 



Regarding \concrete languages, in this draft we restricted to \concrete regular languages, a simple class of languages for which we have shown nice closure properties. It is yet to be determined if \concrete languages mimic a similar language hierarchy as their discrete counterparts.


\section*{Acknowledgements}

This work was supported by the European Research Council (ERC StG DeepSPIN 758969). 
I would like to thank Wilker Aziz, who suggested the idea of sampling from a sparsemax-Gaussian distribution, Vlad Niculae, who was involved in initial discussions, Tim Vieira, who answered several questions about order statistics, Sam Power, who pointed out to manifold stratification, and Juan Bello-Rivas, who suggested the name ``mixed random variables.'' 
This manuscript benefitted from valuable feedback from Wilker Aziz, António Farinhas, Tim Vieira, and the DeepSPIN team. 



\bibliographystyle{apalike}
\bibliography{references}

\appendix

\section{Proof of Correctness of Direct Sum Measure}\label{sec:measure}


We start by recalling the definitions of $\sigma$-algebras, measures, and measure spaces. 
A \textbf{$\sigma$-algebra} on a set $X$ is a collection of subsets, $\Omega \subseteq 2^{X}$, which is closed under complements and under countable unions. 
A \textbf{measure} $\mu$ on $(X, \Omega)$ is a function from $\Omega$ to $\mathbb{R} \cup \{\pm \infty\}$ satisfying (i) $\mu(A) \ge 0$ for all $A \in \Omega$, (ii) $\mu(\varnothing) = 0$, and (iii) the $\sigma$-additivity property: $\mu(\sqcup_{j\in\mathbb{N}} A_j) = \sum_{j\in\mathbb{N}} \mu(A_j)$ for every countable collections $\{A_j\}_{j\in\mathbb{N}} \subseteq \Omega$ of pairwise disjoint sets in $\Omega$. 
A \textbf{measure space} is a triple $(X, \Omega, \mu)$ where $X$ is a set, $\mathcal{A}$ is a $\sigma$-algebra on $X$ and $\mu$ is a measure on $(X, \mathcal{A})$. 
An example is the Euclidean space $X=\mathbb{R}^K$ endowed with the Lebesgue measure, where $\Omega$ is the Borel algebra generated by the open sets (\textit{i.e.} the set $\Omega$ which contains these open sets and countably many Boolean operations over them). 

The correctness of the direct sum measure $\mu^\oplus$ comes from the following more general result, which appears (without proof) as exercise I.6 in \citet{conway2019course}.
\begin{lemma}\label{lemma:direct_sum_measure}
Let $(X_k, \Omega_k, \mu_k)$ be measure spaces for $k=1, \ldots, K$. 
Then, $(X, \Omega, \mu)$ is also a measure space, with $X = \bigoplus_{k=1}^K X_k = \prod_{k=1}^K X_k$ (the direct sum or Cartesian product of sets $X_k$), $\Omega = \{A \subseteq X \mid A 
\cap X_k \in \Omega_k,\,\, \forall k \in [K]\}$, and $\mu(A) = \sum_{k=1}^K \mu_k(A \cap X_k)$.
\end{lemma}
\begin{proof}
First, we show that $\Omega$ is a $\sigma$-algebra. We need to show that (i) if $A \in \Omega$, then $\bar{A} \in \Omega$, and (ii) if $A_i \in \Omega$ for each $i \in \mathbb{N}$ then $\bigcup_{i \in \mathbb{N}} A_i \in \Omega$. 
For (i), we have that, if $A \in \Omega$, then we must have $A \cap {X}_k \in \Omega_k$ for every $k$, and therefore $\bar{A} \cap {X}_k = {X}_k \setminus A = {X}_k \setminus (A \cap {X}_k) \in \Omega_k$, since $\Omega_k$ is a $\sigma$-algebra on ${X}_k$. This implies that $\bar{A} \in \Omega$. 
For (ii), we have that, if $A_i \in \Omega$, then we must have $A_i \cap {X}_k \in \Omega_k$ for every $i\in \mathbb{N}$ and $k\in[K]$, and therefore $\left(\bigcup_{i \in \mathbb{N}} A_i\right) \cap {X}_k = \bigcup_{i \in \mathbb{N}} (A_i \cap {X}_k) \in \Omega_k$, since $\Omega_k$ is closed under countable unions. This implies that $\bigcup_{i \in \mathbb{N}} A_i\in \Omega$. 
Second, we show that $\mu$ is a measure. We clearly have $\mu(A) = \sum_{k=1}^K \mu_k(A \cap {X}_k) \ge 0$, since each $\mu_k$ is a measure itself, and hence it is non-negative. 
We also have $\mu(\varnothing) = \sum_{k=1}^K \mu_k(A \cap \varnothing) = \sum_{k=1}^K \mu_k(\varnothing)  = 0$. Finally, if $\{A_j\}_{j\in\mathbb{N}} \subseteq \Omega$ is a countable collection of disjoint sets, we have $\mu(\sqcup_{j\in\mathbb{N}} A_j) = \sum_{k=1}^K \mu_k(\sqcup_{j\in\mathbb{N}} (A_j \cap {X}_k)) = \sum_{k=1}^K \sum_{j\in\mathbb{N}} \mu_k(A_j \cap {X}_k) = \sum_{j\in\mathbb{N}} \sum_{k=1}^K  \mu_k(A_j \cap {X}_k) = \sum_{j\in\mathbb{N}}   \mu(A_j)$.
\end{proof}

We have seen in \eqref{eq:disjoint_union_faces} that the simplex $\triangle_{K-1}$ can be decomposed as a disjoint union of the relative interior of its faces. Each of these relative interiors is an open subset of an affine subspace isomorphic to $\mathbb{R}^{k-1}$, for $k \in [K]$, which is equipped with the Lebesgue measure for $k>1$ and the counting measure for $k=1$. 
Lemma~\ref{lemma:direct_sum_measure} then guarantees that we can take the direct sum of all these affine spaces as a measure space with the direct sum measure $\mu = \mu^\oplus$ of Definition~\ref{def:direct_sum}.

\section{Proof of Proposition~\ref{prop:regular}}\label{sec:regular}

To show 1, note that if $\mathcal{L}$ is regular, there is an FSA that recognizes it; since a FSA is a particular case of a MFSA, we have that $\mathcal{L}$ is also mixed regular.

The key to prove 2 is to generalize the powerset construction of \citet{rabin1959finite} (which establishes the classic equivalence between deterministic and non-deterministic FSAs). 
The powerset construction creates a deterministic MFSA $(\Sigma, 2^S, \{q\}, G, \gamma)$ from a non-deterministic MFSA $(\Sigma, S, I, F, \delta)$ as follows:
\begin{itemize}
\item Each state of the deterministic MFSA is a subset $P \subseteq S$, the starting state is $q \subseteq S$.
\item The set of final states of the deterministic MFSA is $G=\{P \subseteq S \mid P \cap F \ne \varnothing\}$.
\item The transition function $\gamma: 2^S \times \triangle(\Sigma) \times 2^S \rightarrow \{0,1\}$ is defined as $\gamma(P, \bm{y}, Q) = 1$ iff $Q = \bigcup_{s \in P} \{t\in S \mid \delta(s, \bm{y}, t) = 1\}$.
\end{itemize}
The powerset construction blows up the number of states, but keeps it finite.%
\footnote{Note that, since there is a finite number of states, in a deterministic MFSA, the transitions outgoing a state $s$ form a finite (disjoint) partition of $\triangle(\Sigma)$, as $\triangle(\Sigma) = \bigsqcup_{t \in S} \mathrm{supp}(p_Y(\cdot \mid s,t))$. In a non-deterministic MFSA, we only have a finite union $\triangle(\Sigma) \supseteq \bigcup_{t \in S} \mathrm{supp}(p_Y(\cdot \mid s,t))$. When we convert from a non-deterministic MFSA into a deterministic MFSA, we create new partitions which involve the intersection of these support sets. Each of these support sets belong to the $\sigma$-algebra associated with the direct sum measure on the simplex, and therefore their intersections also belong to that $\sigma$-algebra, which makes them measurable.} %

To prove 3, note that the skeleton of a mixed regular language associated to a MFSA can be obtained by deleting the weights in the transition function and relabeling each edge $\bm{y}$ to $\mathrm{supp}(\bm{y})$, which turns that MFSA into a non-deterministic FSA over $2^\Sigma$. From point 2, we have that this must be a regular language. 

To show 4, we start from the MFSA corresponding to the mixed regular language $\mathcal{L}$, delete the weights in the transition function and, for any edge  $\bm{y}$, create $|\mathrm{supp}(\bm{y})|$ many edges, obtaining a non-deterministic FSA that corresponds the projection of  $\mathcal{L}$; hence this projection is a regular language.

Finally, let us prove 5. 
To build an MFSA that recognizes the negation of $\mathcal{L}$, take an MFSA for $\mathcal{L}$ and toggle final and non-final states. 
To build an MFSA that recognizes the union of $\mathcal{L}_1$ and $\mathcal{L}_2$, take their MFSAs and consider the starting states the union of the two sets of starting states. This results on a non-deterministic MFSA that can be determinized using the powerset construction used to prove item 2.
To build an MFSA that recognizes the intersection  $\mathcal{L}_1$ and $\mathcal{L}_2$ use the two constructions above and apply the de Morgan rules.
To build an MFSA that recognizes the concatenation of $\mathcal{L}_1$ and $\mathcal{L}_2$ take the final states of the MFSA of $\mathcal{L}_1$ (removing them as final states) and add $\epsilon$-transitions to the starting state of the MFSA of $\mathcal{L}_2$. Then apply $\epsilon$-removal and convert the non-deterministic MFSA into a deterministic MFSA with the powerset construction.

\end{document}